\documentclass[letterpaper,10 pt,conference]{ieeeconf}  
\IEEEoverridecommandlockouts                             

\usepackage{amsmath,amssymb,amstext}
\usepackage{graphicx}
\graphicspath{{figures_control/}{./}}
\usepackage{algorithm}
\usepackage{algpseudocode}
\usepackage{amsfonts}
\usepackage{booktabs}
\usepackage{xcolor}
\usepackage{cite}
\usepackage{float}
\usepackage[colorlinks=true,linkcolor=blue,citecolor=blue]{hyperref}
\usepackage{ifthen,amsmath}

\newtheorem{assumption}{Assumption}
\newtheorem{problem}{Problem}

\newtheorem{remark}{Remark}

\newtheorem{lemma}{Lemma}

\newcommand{\sg}[1]{\operatorname{sg}\!\left[#1\right]}

\title{\LARGE \bf
Learning Sampled-data Control for Swarms via MeanFlow
}

\author{Anqi Dong, Yongxin Chen, Karl H. Johansson, Johan Karlsson
\thanks{Anqi Dong is with the Department of Decision and Control Systems, Department of Mathematics, and the Digital Future, KTH Royal Institute of Technology, SE-114 28 Stockholm, Sweden {\tt\small anqid@kth.se}}%
\thanks{Yongxin Chen with the Institute for Robotics and Intelligent Machines, Georgia Institute of Technology, Atlanta, GA 30332, USA {\tt\small yongchen@gatech.edu}}%
\thanks{Karl Johansson with the Department of Decision and Control Systems and the Digital Future, KTH Royal Institute of Technology, Lindstedtsv\"agen 25, SE-114 28 Stockholm, Sweden {\tt\small kallej@kth.se}}%
\thanks{Johan Karlsson with the Department of Mathematics and the Digital Future, KTH Royal Institute of Technology, Lindstedtsv\"agen 25, SE-114 28 Stockholm, Sweden {\tt\small johan.karlsson@math.kth.se}}%
\thanks{This research has been supported in part by the Swedish Research Council Distinguished Professor Grant 2017-01078, Knut and Alice Wallenberg Foundation Wallenberg Scholar Grant, and the Swedish Research Council (VR) under grant 2020-03454, KTH Digital Futures. The computations were enabled by resources provided NAISS, partially funded by the Swedish Research Council grant 2022-06725.}
}

\begin{document}

\maketitle
\thispagestyle{empty}
\pagestyle{empty}

\begin{abstract}


Steering large-scale swarms with only limited control updates 
is often needed due to communication or computational constraints,
yet most learning-based approaches do not account for this and instead model instantaneous velocity fields. As a result, the natural object for decision making is a finite-window control quantity rather than an infinitesimal one. To address this gap, we consider the recent machine learning framework MeanFlow and generalize it to the setting with general linear dynamic systems. This results in a new sampled-data learning framework that operates directly in control space and that can be applied for swarm steering. To this end, we learn the finite-horizon coefficient that parameterizes the minimum-energy control applied over each interval, and derive a differential identity that connects this quantity to a local bridge-induced supervision signal. This identity leads to a simple stop-gradient regression objective, allowing the interval coefficient field to be learned efficiently from bridge samples. The learned policy is deployed through sampled-data updates, guaranteeing that the resulting controller exactly respects the prescribed linear time-invariant dynamics and actuation channel. The resulting method enables few-step swarm steering at scale, while remaining consistent with the finite-window actuation structure of the underlying control system.
\end{abstract}


\section{Introduction}

Large-scale swarm systems arise in robotics, autonomous transportation, distributed sensing, and collective exploration \cite{brambilla2013swarm, chung2018survey, duan2023animal, zhou2022swarm}. In such settings, the objective is to steer an entire population toward a desired configuration or distribution, rather than to plan individual trajectories agent by agent. For large-scale swarm and wide operating region, this viewpoint of control is naturally distributed, as it seeks control policies that drive an initial law to a target law in a dynamically consistent manner \cite{chen2023density, haasler2021control, ringh2023mean}. In practice, swarm control is implemented in sampled-data form \cite{ackermann2012sampled}. Inputs are updated at discrete times and then applied over finite intervals. This distinction is negligible when update frequency of controller is high, but becomes crucial in few-step control, where each control action must remain in effect over a larger time interval \cite{bouffanais2016design,chang2025efficient}.

Recent fast developing Flow-based generative models face a closely related issue. These methods transform a reference distribution into a target distribution through time-dependent dynamics over a prescribed horizon. The learned object is typically instantaneous, such as velocity fields in flow matching \cite{lai2025principles,lipman2022flow} or score field in diffusion and score-based models \cite{dockhorn2021score,elamvazhuthi2025score,song2020score}. The training process samples an intermediate time and state, and regresses an instantaneous label induced by a training bridge \cite{albergo2023stochastic,liu2022flow}, while execution takes place over finite windows. When many small steps are available, this mismatch remains limited, whereas under few-step budget it becomes fundamental, since coarse discretization and repeated composition amplify local errors \cite{chen2023probability,lu2022dpm,salimans2022progressive,song2020denoising,zhang2022fast,zhang2023gddim}.

MeanFlow \cite{geng2025mean} addresses this mismatch by learning a window-level quantity rather than an infinitesimal one. Instead of modeling instantaneous velocity alone, it introduces an average velocity over each time window, namely the quantity actually used by coarse-time updates. This viewpoint has since been further studied from the perspectives of stability and robustness under coarse composition \cite{geng2025improved,hu2025cmt}. The key idea is that when execution is window-based, the learned object should match the quantity applied over that window.

In control theory, this principle may be adapted to structural dynamics. In many problems, admissible motion is not specified by an arbitrary vector field, but by a state equation in which homogeneous dynamics and fixed actuation channels, and the learnable part enters only through the control input \cite{brockett2015finite,kalman1959unified}. For sampled-data swarm control under linear dynamics \cite{chen2012optimal,francis2002stability}, this suggests that the relevant window-level object is not an averaged state velocity, but a control-space quantity associated with the finite-horizon response of the system. A natural choice is the minimum-energy control over each interval. Indeed, for a controllable linear system, the finite-horizon steering problem admits a closed-form minimum-energy control, determined by a coefficient chosen once per interval. Accordingly, the learned object should be aligned with this sampled-data implementation.

\begin{figure}
\centering
\includegraphics[width=1\linewidth]{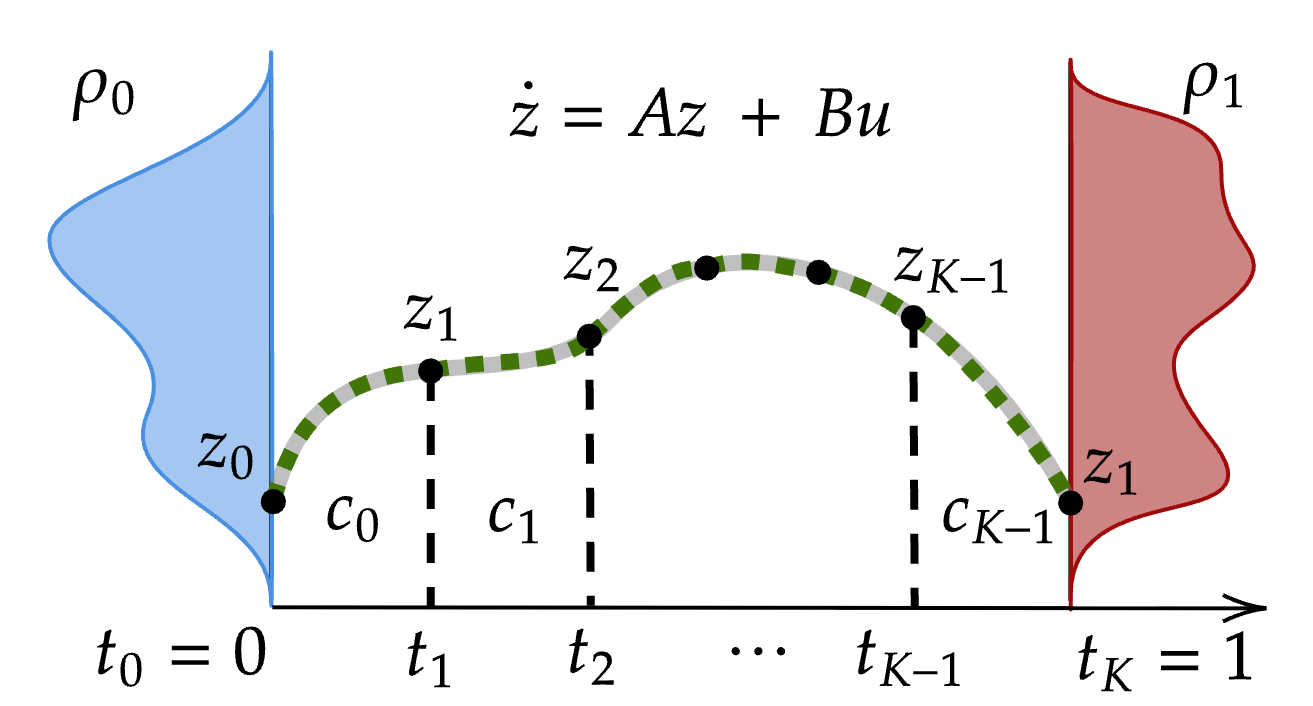}
\caption{Swarm distribution evolves from an initial law $\rho_0$ at $t_0=0$ to a target law $\rho_1$ at $t_K=1$ under the linear dynamics $\dot z=Az+Bu$. The proposed model learns an interval coefficient $c_k=c_\theta(z_k,t_k,t_{k+1})$, which generates the finite-horizon minimum-energy control over $[t_k,t_{k+1}]$. 
}
\vspace{-0.2in}
\label{fig:illu}
\end{figure}

To this end, we propose a framework for sampled-data swarm control under linear time-invariant dynamics in the spirit of MeanFlow. Specifically, on each interval, we learn a finite-horizon minimum-energy control by its interval coefficient. The resulting supervision target is available explicitly from the sampled bridge endpoints through the controllability Gramian. In implementation, the learned controller is applied through sampled-data window updates, so the prescribed dynamics and actuation map are preserved by construction. In this way, we place MeanFlow in a control framework and provide a scalable approach to few-step swarm steering under sampled-data linear dynamics \cite{chen2012optimal,francis2002stability}.

The organization of the rest of the paper is as follows. Section~\ref{sec:prelim} introduces the preliminaries of flow matching models. Section~\ref{sec:swarm-meanflow} formulates the sampled-data MeanFlow problem. Section~\ref{sec:algorithms} presents learning and implementation. Section~\ref{sec:experiments} provides numerical illustrations, and Section~\ref{sec:discussion} concludes the paper.

\section{Recap: Flow matching and MeanFlow}\label{sec:prelim}

We begin by recalling the learning framework underlying flow matching and MeanFlow. Both methods are built on bridge-based supervision. Rather than learning directly from a full trajectory distribution, one starts from two endpoint distributions and specifies how an intermediate state is generated between paired samples from these endpoints. This bridge construction converts the generative problem into a supervised one, because each intermediate point comes equipped with a target quantity induced by the bridge. In flow matching, this target is instantaneous and takes the form of a velocity at a single time. MeanFlow modifies this viewpoint by replacing the instantaneous object with one associated with a finite time interval, thereby aligning the learned quantity more closely with coarse-time execution.

Let $\rho_{\mathrm{data}}$ and $\rho_{\mathrm{base}}$ be two distributions on $\mathbb R^d$, and let $z_1 \sim \rho_{\mathrm{data}}$ and $z_0 \sim \rho_{\mathrm{base}}$. 
A bridge specifies how intermediate states are generated from these endpoint samples. Formally, it is a family of maps $\{\mathcal B_t\}_{t\in[0,1]}$ that induces a differentiable trajectory
$$
z_t=\mathcal B_t(z_0,z_1),
$$
with instantaneous velocity
$
v_t=\dot z_t=\frac{d}{dt}\mathcal B_t(z_0,z_1).
$
In this way, each sampled triple $(z_0,z_1,t)$ determines a supervised pair $(z_t,v_t)$. A standard example is the linear interpolation
$$
z_t=(1-t)z_0+t z_1,
$$
with constant velocity $v_t=z_1-z_0$ along the path.

Flow matching then seeks to learn a time-dependent vector field from the supervised samples induced by the bridge \cite{lipman2022flow,mei2024flow,yue2025oat}. Let $v_\theta(z,t)$ be a parametric velocity field. The standard training objective is then
\begin{align*}
\min_\theta \;
\mathbb E_{z_0,z_1,t}
\bigl\|v_\theta(z_t,t)-v_t\bigr\|^2
\quad
\text{s.t. } z_t=\mathcal B_t(z_0,z_1),
\end{align*}
where the expectation is taken over {$z_1 \sim \rho_{\mathrm{data}}$ and $z_0 \sim \rho_{\mathrm{base}}$}, 
$z_1\sim \rho_{\mathrm{base}}$, and $t\sim \mathcal U(0,1)$.
In this formulation, the learned object is instantaneous, since the supervision specifies the velocity only at the single time $t$ rather than over a finite time interval.

MeanFlow starts from the observation that, in practice, the learned dynamics are often executed over a finite time interval rather than in the infinitesimal-step limit \cite{geng2025mean}. It suggests that the quantity to be learned should also be tied to a finite window. In the unstructured setting, MeanFlow therefore replaces the instantaneous velocity target by a window-level object, namely the forward average velocity over the interval $[t,r]$,%
\footnote{The original MeanFlow is presented in a backward-time convention. Here we adopt a forward-time convention, so the finite-window object is defined on $[t,r]$ with $r>t$.}
$$
u(z_t,t,r):=\frac{1}{r-t}\int_t^r v(z_\tau,\tau)\,d\tau,
\qquad r>t.
$$
This quantity summarizes the average velocity from the current state $z_t$ to the future state $z_r$, and is therefore more naturally aligned with finite-step propagation than the instantaneous velocity $v(z_t,t)$.

For fixed $r$, differentiating this expression with respect to $t$ yields an identity that relates the window-level object to local quantities along the bridge. This leads to the training target
$$
u_{\mathrm{tgt}}
=
v(z_t,t)
+
(r-t)\Big[
\partial_z \bar u_\theta(z_t,t,r)\,v(z_t,t)
+
\partial_t \bar u_\theta(z_t,t,r)
\Big],
$$
which replaces the instantaneous supervision used in standard flow matching with supervision defined over the interval $[t,r]$. The corresponding training regression objective is
\begin{align*}
\min_\theta \;
\mathbb E_{z_0,z_1,t,r}
\Bigl\|
\bar u_\theta(z_t,t,r)
-
\sg\!\bigl(u_{\mathrm{tgt}}(z_t,t,r)\bigr)
\Bigr\|^2,
\end{align*}
where the expectation is taken over endpoint samples together with times $t<r$, and $\sg{\cdot}$ denotes stop-gradient operator \cite{geng2025mean}.

The essential point is that MeanFlow no longer learns only the local bridge velocity at a single time. Instead, it learns an object associated with a finite interval, which makes the learned dynamics more compatible with coarse-time execution and less sensitive to discretization error over long steps. This interval-based viewpoint is the one we will adapt in the controlled setting in the next section.

\section{Swarm control via MeanFlow}\label{sec:swarm-meanflow}

We consider a large-scale swarm whose configuration at time $t$ is described, at the macroscopic level, by a probability distribution $\rho_t$ on $\mathbb R^d$. The objective is to guide this distribution from a prescribed initial configuration $\rho_0$ to a target configuration $\rho_1$. Rather than designing separate controls for individual agents, we seek a description at the level of the evolving ensemble, where the collective motion is shaped through an underlying controlled dynamics.

At the level of a single state, we consider the linear system
\begin{align}\label{eq:lin-sys}
\dot z(\tau)=Az(\tau)+Bu(\tau),
\qquad \tau\in[0,1],
\end{align}
where $A\in\mathbb R^{d\times d}$ and $B\in\mathbb R^{d\times m}$ are fixed matrices, and $u(\tau)\in\mathbb R^m$ is the control input. The question that concerns us is the following: given two times $t<r$ and two states $z_t,z_r\in\mathbb R^d$, how should one drive the system from $z_t$ to $z_r$ over the interval $[t,r]$ in the most economical way.

This leads to the familiar finite-horizon minimum-energy steering problem
\begin{align}\label{eq:min-energy-problem}
\begin{aligned}
\min_{u(\cdot)}\qquad & \int_t^r \|u(\tau)\|^2\,d\tau \\
\text{subject to}\quad
& \dot z(\tau)=Az(\tau)+Bu(\tau),\\
&z(t)=z_t,\quad z(r)=z_r.
\end{aligned}
\end{align}
The problem is classical, but here it plays a central conceptual role. It tells us what the natural finite-window object should be in the controlled setting: not an instantaneous velocity, but the quantity that parameterizes the least-effort transfer across the entire interval.

Then the classical  variation-of-constants formula gives
\begin{align}\label{eq:voc}
z_r=\Phi(r,t)z_t+\int_t^r \Phi(r,\tau)B\,u(\tau)\,d\tau,
\end{align}
where
$\Phi(r,t):=e^{A(r-t)}$
denotes the state transition matrix.


\begin{assumption}[Controllability]\label{ass:controllability}
The pair $(A,B)$ is controllable. Equivalently, for every $r>t$, the Gramian 
\begin{align}\label{eq:gramian}
W(t,r):=\int_t^r \Phi(r,\tau)BB^\top\Phi(r,\tau)^\top\,d\tau
\end{align}
is nonsingular.
\end{assumption}


Under Assumption~\ref{ass:controllability}, every pair of states $(z_t,z_r)$ can be connected over $[t,r]$, and problem \eqref{eq:min-energy-problem} admits a unique minimizer. More precisely, the optimal control is of the form
\begin{align}\label{eq:u-star}
u^\star(\tau\mid t,r)=B^\top\Phi(r,\tau)^\top c(z_t,t,r),
\qquad \tau\in[t,r],
\end{align}
where, along the bridge, the coefficient $c(z_t,t,r)\in\mathbb R^d$ is given by
\begin{align}\label{eq:c-def}
c(z_t,t,r)
=
W(t,r)^{-1}\bigl(z_r-\Phi(r,t)z_t\bigr).
\end{align}
Substituting \eqref{eq:u-star} into \eqref{eq:voc}, one recovers
\begin{align}\label{eq:endpoint-representation}
z_r=\Phi(r,t)z_t+W(t,r)c(z_t,t,r).
\end{align}
Thus, the entire optimal transfer over the interval $[t,r]$ is encoded by the single vector $c(z_t,t,r)$. This coefficient summarizes, in a compressed yet exact form, the steering action required to move from the present state to the future one. It is therefore the natural analogue, in the controlled linear setting, of the window-level object that MeanFlow seeks to learn.

We now place this construction in the bridge-based training framework. Let $t\mapsto z_t$ be a differentiable bridge trajectory connecting samples from the endpoint distributions. For a sampled pair of times $t<r$, the states $z_t$ and $z_r$ lie on the same bridge trajectory, so the future state $z_r$ is determined once the bridge and the forward window are fixed. For this reason, the corresponding finite-horizon steering coefficient is naturally associated with the current bridge state and the interval $[t,r]$, and we denote it by $c(z_t,t,r)$. It satisfies
\begin{align}\label{eq:c-target-core}
W(t,r)c(z_t,t,r)=z_r-\Phi(r,t)z_t.
\end{align}
This is the finite-horizon quantity we seek to learn.



The coefficient $c(z_t,t,r)$ also admits an integral representation that reveals its dynamical meaning. If the bridge evolves according to
$$
\dot z_\tau=v(z_\tau,\tau),
$$
then, by combining \eqref{eq:voc} and \eqref{eq:c-target-core}, we obtain
\begin{align}\label{eq:c-integral-identity}
W(t,r)c(z_t,t,r)
=
\int_t^r \Phi(r,\tau)\bigl(\dot z_\tau-Az_\tau\bigr)\,d\tau.
\end{align}
If, moreover, the bridge dynamics are expressed through the input channel as
$
\dot z_\tau=Az_\tau+Bv(z_\tau,\tau),
$
then the identity reduces to
\begin{align}\label{eq:c-integral-identity-input}
W(t,r)c(z_t,t,r)
=
\int_t^r \Phi(r,\tau)B\,v(z_\tau,\tau)\,d\tau.
\end{align}
These relations show that the coefficient does not merely describe an endpoint discrepancy. It gathers, over the entire interval, the cumulative action needed to produce the transfer. In this sense, it is the proper finite-window object for the controlled problem, and the one that naturally takes the place of the averaged velocity in the original MeanFlow construction.

While $c(z_t,t,r)$ is defined through a finite-horizon relation, it also satisfies a local identity along the bridge. This is what allows the interval coefficient to be learned from local information. Starting from \eqref{eq:c-integral-identity-input}, we differentiate with respect to the left endpoint $t$. Since
$$
\frac{d}{dt}W(t,r)
=
-\,\Phi(r,t)BB^\top \Phi(r,t)^\top,
$$
we obtain
\begin{align}\label{eq:c-differential-identity}
\begin{aligned}
W(t,r)\frac{d}{dt}c(z_t,t,r)
&-\Phi(r,t)BB^\top \Phi(r,t)^\top c(z_t,t,r)\\
=
&-\Phi(r,t)B\,v(z_t,t).    
\end{aligned}
\end{align}
This relation has a clear interpretation. Although $c(z_t,t,r)$ summarizes the steering action over the entire interval $[t,r]$, its evolution with respect to the current time $t$ is governed by the local bridge velocity at $z_t$. In this way, the finite-window object remains accessible through local supervision.

Motivated by this identity, we introduce a parametric field
$$
c_\theta(z_t,t,r)\in\mathbb R^d,
$$
which is intended to approximate the forward finite-horizon steering coefficient along sampled bridge trajectories. The training objective is chosen so that $c_\theta$ satisfies \eqref{eq:c-differential-identity} in the least-squares sense.

\begin{problem}\label{prob:linear-meanflow}
Consider the linear system \eqref{eq:lin-sys} with fixed matrices $A$ and $B$. Under Assumption~\ref{ass:controllability}, learn a parametric field $c_\theta(z_t,t,r)$ by minimizing\footnote{In implementation, the target term is detached through a stop-gradient operator, exactly as in MeanFlow, so that the optimization is carried out with a stable one-sided regression objective.}
\begin{align*}
\mathcal L(\theta)
:=
\mathbb E_{z_0,z_1,t,r}
\bigl\|
R_\theta(z_t,t,r)
\bigr\|^2,
\end{align*}
where the residual is defined by
\begin{align*}
&R_\theta(z_t,t,r):=
\Phi(r,t)BB^\top \Phi(r,t)^\top c_\theta(z_t,t,r)\\
&\phantom{xxxxxx}-\sg{\Bigl(
W(t,r)\frac{d}{dt}c_\theta(z_t,t,r)
+
\Phi(r,t)B\,v(z_t,t)
\Bigr)},
\end{align*}
and $(z_t,z_r)$ are sampled along the training bridge with $t<r$. 
\end{problem}

Problem~\ref{prob:linear-meanflow} places the learning problem directly in coefficient space. The object being learned is not an instantaneous velocity in state space, but the forward finite-horizon coefficient that parameterizes the minimum-energy steering action over the interval $[t,r]$. In this way, the model learns the natural MeanFlow object associated with the linear controlled dynamics.

The next observation shows that this construction is fully consistent with the original MeanFlow viewpoint, and that it reduces, in the appropriate limiting case, to the familiar instantaneous supervision of Flow Matching \cite{liu2022flow}.

\begin{lemma}\label{lem:recover-fm-meanflow}
When $A=0$ and $B=I$, Problem~\ref{prob:linear-meanflow} reduces to the forward MeanFlow formulation. Moreover, if the bridge trajectory is differentiable, then in the limit $r\to t^+$ the finite-window target converges to the standard Flow Matching target \cite{mei2024flow}.
\end{lemma}

\begin{proof}
Consider first the case $A=0$ and $B=I$. Then $\Phi(r,t)=I$ and
$
W(t,r)=\int_t^r I\,d\tau=(r-t)I.
$
Accordingly, the defining relation \eqref{eq:c-def} becomes
$
(r-t)c(z_t,t,r)=z_r-z_t,
$
and hence
$$
c(z_t,t,r)=\frac{z_r-z_t}{r-t}.
$$
This is precisely the forward MeanFlow target over the interval $[t,r]$.

Assume the bridge trajectory is differentiable, then it has
$
\frac{z_r-z_t}{r-t}\to \dot z_t
$
as $r\to t^+$. Therefore,
$
\lim_{r\to t^+} c(z_t,t,r)=\dot z_t.
$
Thus, in the infinitesimal-window limit, the finite-horizon target reduces to the instantaneous bridge velocity, which is exactly the supervision target used in Flow Matching.
\end{proof}

Notably, the formulation is not restricted to first-order swarm kinematics. The state variable $z$ may include higher-order components. For instance, in a second-order model one may write
$
z=
\begin{bmatrix}
x &v
\end{bmatrix}^\top,
$
so that the control enters through an acceleration channel. In that case, $c(z_t,t,r)$ represents the coefficient of the corresponding finite-horizon minimum-energy control over the interval $[t,r]$. 

More broadly, the same construction extends without change to higher-order linear state-space models, as well as to bridge trajectories induced by optimal transport couplings between the initial and target swarm distributions \cite{tong2024improving}.

\section{Sampled-data interpretation}
The coefficient learned by the model has a direct operational meaning. Once a current state $z_t$ and a future time $r>t$ are given, the field $c_\theta(z_t,t,r)$ determines the control profile to be applied over the entire interval $[t,r]$. Specifically, the corresponding control signal is
\begin{align}\label{eq:u-star-implemented}
u(\tau)=B^\top \Phi(r,\tau)^\top c_\theta(z_t,t,r),
\qquad \tau\in[t,r],
\end{align}
and the induced state update is
\begin{align}\label{eq:gramian-window-update}
z_r=\Phi(r,t)z_t+W(t,r)c_\theta(z_t,t,r).
\end{align}
Thus, the learned coefficient is not merely a training target. It is precisely the quantity consumed by the finite-horizon sampled-data update over the window $[t,r]$.

The exact interval coefficient also possesses a natural composition property across adjacent windows. This shows that the finite-horizon steering action is consistent under temporal subdivision.

\begin{lemma}[Additivity]\label{lem:additivity}
Let $t<s<r$. Then along any trajectory of \eqref{eq:lin-sys},
\begin{align*}
&W(t,r)c(z_t,t,r)\\
&=\Phi(r,s)W(t,s)c(z_t,t,s)
+
W(s,r)c(z_s,s,r).
\end{align*}
\end{lemma}

\begin{proof}
By definition,
$$
W(t,r)c(z_t,t,r)=z_r-\Phi(r,t)z_t.
$$
Introduce the intermediate state $z_s$, and use the semigroup property
$$
\Phi(r,t)=\Phi(r,s)\Phi(s,t)
$$
to write
\begin{align*}
z_r-\Phi(r,t)z_t
&=
z_r-\Phi(r,s)z_s
+
\Phi(r,s)\bigl(z_s-\Phi(s,t)z_t\bigr).
\end{align*}
Applying the definition of the interval coefficient on the subintervals $[t,s]$ and $[s,r]$ yields
\begin{align*}
z_r\!-\!\Phi(r,t)z_t
=
W(s,r)c(z_s,s,r)\!
+\!
\Phi(r,s)W(t,s)c(z_t,t,s),
\end{align*}
which proves the claim.
\end{proof}

Lemma~\ref{lem:additivity} shows that the steering action over a long interval may be decomposed into two successive shorter transfers without losing consistency with the dynamics. In this sense, the coefficient provides a dynamically coherent window-level description of the controlled evolution.

When $c_\theta$ does not exactly coincide with the true interval coefficient, the discrepancy appears directly at the level of the sampled-data update. Indeed, one may write
\begin{align}\label{eq:gramian-residual}
z_r
=
\Phi(r,t)z_t+W(t,r)c_\theta(z_t,t,r)+\eta(t,r),
\end{align}
where
$$
\eta(t,r)=W(t,r)\bigl(c(z_t,t,r)-c_\theta(z_t,t,r)\bigr).
$$
Thus, the state-space error over the interval is exactly the coefficient approximation error through the Gramian. In particular, the residual vanishes when the learned field matches the exact finite-horizon steering coefficient.

\section{Control learning and implementation}\label{sec:algorithms}

We now illustrate how the learned interval-coefficient field is used in practice. The roles of training and sampling are different. During training, a differentiable bridge provides endpoint states from which the finite-horizon steering coefficient is constructed. During sampling, the bridge is no longer used. One evaluates the learned field on each time window and applies the resulting finite-horizon control through the corresponding sampled-data update.

\subsection{Interval control coefficient learning}\label{subsec:zoh-training}

Training uses only bridge information over sampled forward windows. For each sampled pair $(t,r)$ with $t<r$, the bridge provides the states $z_t$ and $z_r$, together with the local bridge velocity at time $t$. The network does not learn an instantaneous control signal directly. Instead, it learns the coefficient that determines the minimum-energy control over the entire interval $[t,r]$. Accordingly, supervision is imposed through the local differential identity satisfied by the forward finite-horizon steering coefficient, rather than by direct regression onto an explicit teacher coefficient.

For training, we use the minimum-energy bridge induced by the linear dynamics between sampled swarms with endpoint states $z_0 \sim \rho_0$ and $z_1 \sim \rho_1$. Specifically, for $\tau\in[0,1]$, the bridge is given by
\begin{align}\label{eq:bridge}
\scalebox{0.95}{$z_\tau
=
\Phi(\tau,0)z_0
+
W(0,\tau)\Phi(1,\tau)^\top W(0,1)^{-1}
\bigl(z_1-\Phi(1,0)z_0\bigr).$}
\end{align}
This trajectory is the finite-horizon minimum-energy path connecting $z_0$ and $z_1$ under the linear dynamics. Given a sampled forward window $[t,r]\subset[0,1]$, the corresponding bridge states $z_t$ and $z_r$ are obtained by evaluating \eqref{eq:bridge} at the times $t$ and $r$.

Differentiating \eqref{eq:bridge}, or equivalently using the bridge dynamics induced by the minimum-energy control, yields
\begin{align*}
\dot z_t
=
Az_t
+
BB^\top \Phi(1,t)^\top W(0,1)^{-1}
\bigl(z_1-\Phi(1,0)z_0\bigr).
\end{align*}
Hence, under the input-channel representation
$
\dot z_t = Az_t + Bv(z_t,t),
$
with corresponding bridge action satisfies
\begin{align*}
Bv(z_t,t)
=
BB^\top \Phi(1,t)^\top W(0,1)^{-1}
\bigl(z_1-\Phi(1,0)z_0\bigr).   
\end{align*}
Algorithm~\ref{alg:zoh-train} summarizes the resulting training procedure.

\begin{algorithm}[t]
\caption{Training the interval-coefficient model $c_\theta$}
\label{alg:zoh-train}
\begin{algorithmic}[1]
\Require Initial and target swarm distributions $\rho_0$ and $\rho_1$, system matrices $A$ and $B$, parametric interval-coefficient model $c_\theta(z,t,r)\in\mathbb R^d$
\While{not converged}
\State Sample endpoint states and training window
\[
z_0 \sim \rho_0,\;
z_1 \sim \rho_1,\;
t \sim \mathcal U(0,1),\;
r \sim \mathcal U(t,1)
\]
\State Compute the bridge states $z_t$ and $z_r$ in \eqref{eq:bridge}
\State Compute the bridge action at time $t$
\begin{align*}
Bv(z_t,t)
=
BB^\top \Phi(1,t)^\top W(0,1)^{-1}
\bigl(z_1-\Phi(1,0)z_0\bigr)
\end{align*}
\State Evaluate the predicted coefficient $c_\theta(z_t,t,r)$
\State Compute the residual $R_\theta(z_t,t,r)$ from the differential identity
\State Update $\theta$ by one gradient step on $\mathcal L(\theta)$
\EndWhile
\end{algorithmic}
\end{algorithm}

The bridge \eqref{eq:bridge} is used only during training, as a source of finite-window supervision. The network itself learns the coefficient that determines the control action over the entire interval. Since the training signal is imposed through the forward differential identity, the learned object remains window-level even though the loss is local in time.

\subsection{Interval control implementation}\label{subsec:zoh-sampling}

After training, the output is the learned coefficient field $c_\theta(z_t,t,r)$. Starting from a sample drawn from the initial swarm distribution, the swarm state is propagated forward along a discrete time grid. On each interval, the network is queried once to produce a single coefficient, and that coefficient determines the finite-horizon control over the window. The state is then propagated by the corresponding sampled-data update. Algorithm~\ref{alg:zoh-sample} summarizes the resulting swarm evolution.

\begin{algorithm}[t]
\caption{Finite-horizon forward swarm propagation}
\label{alg:zoh-sample}
\begin{algorithmic}[1]
\Require Swarm distribution $\rho_0$, system matrices $A$, $B$
\Require Trained interval-coefficient $c_\theta(z,t,r)\in\mathbb R^d$
\State Choose a time grid $0=t_0<t_1<\cdots<t_K=1$
\State Draw an initial swarm state sample $z_0\sim \rho_0$
\For{$k=0,\dots,K-1$}
\State Set the current time window to $[t_k,t_{k+1}]$
\State Compute predicted coefficient
$c_k=c_\theta(z_k,t_k,t_{k+1})$
\State Propagate swarm state with finite-horizon update
$$z_{k+1}=\Phi(t_{k+1},t_k)z_k+W(t_k,t_{k+1})c_k$$
\EndFor
\State Return swarm trajectory $\{z_k\}_{k=0}^{K}$
\end{algorithmic}
\end{algorithm}

The state update above is equivalent to applying, on each interval $[t_k,t_{k+1}]$, the control signal
\[
u_k(\tau)=B^\top \Phi(t_{k+1},\tau)^\top c_k,
\qquad \tau\in[t_k,t_{k+1}].
\]
Thus, the controller is updated only at the sampling instants, while the within-window is determined analytically by the system matrices and the learned coefficient.

Although the learning formulation is derived from bridge supervision, the implemented controller is not open-loop. The learned field $c_\theta(z,t,r)$ depends on the current state $z$, so the control applied on each interval is recomputed online as $c_k=c_\theta(z_k,t_k,t_{k+1})$.
The resulting implementation is therefore a sampled-data state-feedback controller.

\begin{remark}
Problem \ref{prob:linear-meanflow} provides natural degree of robustness. If the swarm state is perturbed, then the control on the next interval changes accordingly through the state dependence of $c_\theta$. The same mechanism allows the controller to respond to particle loss or stochastic perturbations during propagation, since the control action is generated from the current state.
\end{remark}

\section{Case studies}\label{sec:experiments}
We present two case studies for large-scale swarm control. Our first case considers planar navigation for large-scale ground robots, while the second concerns spatial maneuvering in three dimensions, representative of aerial or underwater swarms. All numerical experiments are implemented in PyTorch and executed on Tesla V100-SXM2-32GB GPU. The code to conduct all experiments can be found at \url{github.com/dytroshut/Meanflow_Control}.

\subsection{Two-dimensional planar swarm navigation}

We first consider two-dimensional planar swarm navigation. The initial swarm distribution at $t=0$ forms the initials of authors' first names, ``AYKJ'', and target distribution at $t=1$ forms authors' last names, ``DCJK''. Fig.~\ref{fig:2d} reports the evolution of the swarm distribution, with snapshots taken at $t=0,\,0.25,\,0.5,\,0.75,$ and $1$. Throughout, we color particles according to their spatial location to visually tracked deformation and the induced correspondence over time.

The first row corresponds to the drift-free case $A=0$ with $B=I$. This setting reduces to a standard MeanFlow-like generation model, where the dynamics are purely control-driven and the learned generation is carried out with $16$ steps. In this case, the swarm deforms in a direct manner, and the induced origin-to-destination correspondence is largely aligned with the horizontal direction in the plotted configuration. The second row introduces a rotational drift. We set $B=I$ and choose
$A=\begin{bmatrix}
0 & -\omega, \ 
\omega & 0
\end{bmatrix}$,
which generates a planar rotation. In our parameter setting, the accumulated drift over the horizon corresponds to a $90^\circ$ rotation. This change in the underlying dynamics reshapes the entire transport process. While the same endpoint distributions are enforced, the induced origin-to-destination correspondence changes with respect to the dynamics. The colored particles that were matched along a horizontal direction in the drift-free case are now matched along a vertical direction, reflecting the rotation imposed by the drift. This example highlights the central point of our framework: \emph{the underlying dynamics alter the geometry of the feasible paths and, consequently, the transport plan that between endpoint distributions}.

\begin{figure*}[htb!]
    \centering
    \includegraphics[width=1\linewidth]{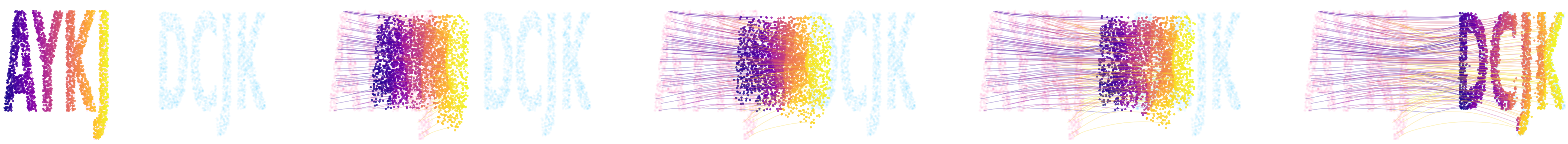}
    \includegraphics[width=1\linewidth]{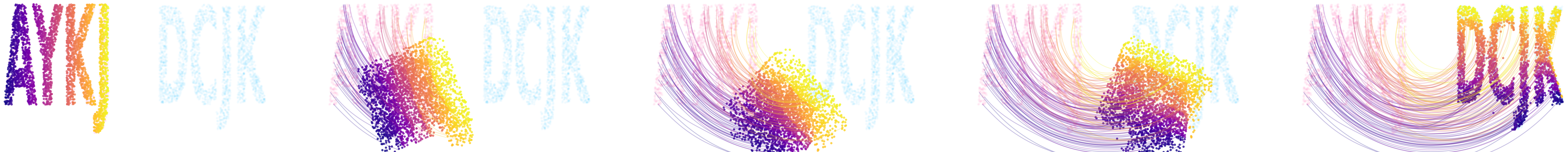}
    \caption{Two-dimensional case: swarm from ``AYKJ'' to ``DCJK''.}
    \label{fig:2d}
\end{figure*}

\subsection{Three-dimensional spatial swarm maneuvering}

Finally, we showcase our model on a geometrically more demanding three-dimensional task (Fig.~\ref{fig:3d}), where the swarm is transported from an initial pyramid to a target torus. This example is challenging in several ways. The swarm must reorganize its spatial support in $\mathbb{R}^3$, carve out a hollow central region to form the torus, and do so through an evolution that remains coherent rather than collapsing into unstable or fragmented motion.

In Fig.~\ref{fig:3d} we compare several underlying dynamics. The first row uses the standard MeanFlow setting (drift-free dynamics). The second and third rows introduce rotational drift restricted to the $x$--$y$ plane, illustrating how planar rotation reshapes the intermediate evolution with same endpoint distributions. The fourth row applies a combined rotational drift in three dimensions. In all cases, the controller is implemented using $16$ steps.

In the drift-free setting and in the $x$--$y$ rotational cases, these tracked particles evolve largely within the expected planes, and the resulting matching retains a relatively direct geometric structure. In contrast, when the rotational drift acts in all three dimensions, the drift fundamentally alters the feasible paths. This leads to a noticeably different intermediate geometry and, consequently, a different final matching pattern.

\begin{figure*}[t]
    \centering
    \includegraphics[width=1\linewidth]{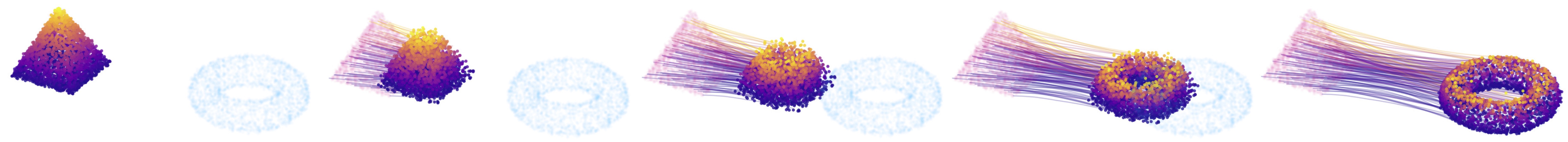}
    \includegraphics[width=1\linewidth]{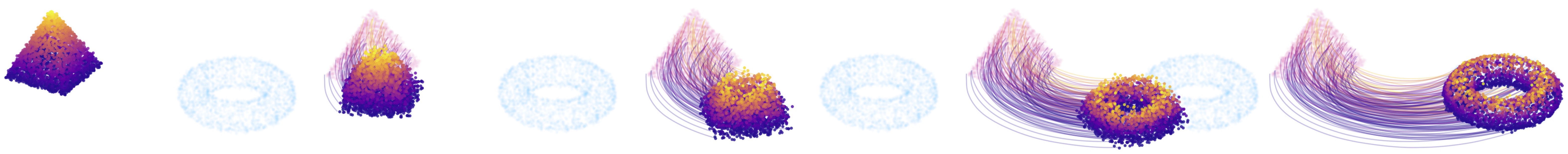}
    \includegraphics[width=1\linewidth]{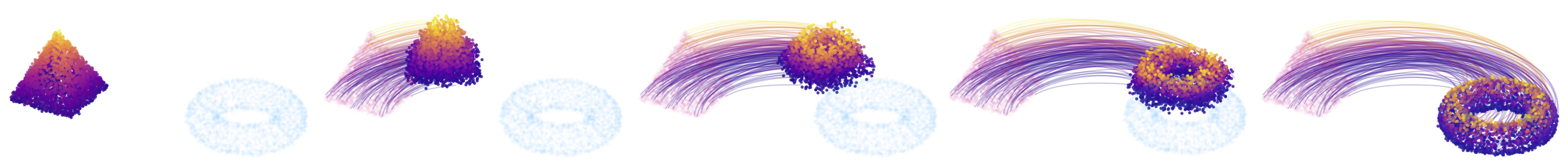}
    \includegraphics[width=1\linewidth]{3d_rotation2.jpg}
    \caption{Three-dimensional case: swarm from pyramid to torus.}
    \label{fig:3d}
\end{figure*}

\section{Conclusion}\label{sec:discussion}

In this paper, we propose a control-space MeanFlow formulation for swarm control under sampled-data linear dynamics. 
In this approach, the control signal is broadcast to all agents at the beginning of each time interval, and each agent then operates autonomously over that interval. The key advantage of mean-flow methods is that they eliminate the need for frequent control updates.
More specifically, the learned object is the time-interval-constant coefficient of the finite-horizon minimum-energy control, rather than an instantaneous state-space field. The resulting supervision target is available explicitly from the sampled bridge endpoints through the controllability Gramian, and the learned coefficient field is used directly in the sampled-data implementation. Extensions to nonlinear dynamics, constrained control, and stability analysis under coarse-time execution remain for future work.

\bibliographystyle{IEEEtran}
\bibliography{references}

\end{document}